\documentclass{article}

     \usepackage[preprint,nonatbib]{neurips_2020}

\usepackage[numbers]{natbib}

\usepackage[utf8]{inputenc} % allow utf-8 input
\usepackage[T1]{fontenc}    % use 8-bit T1 fonts
\usepackage{hyperref}       % hyperlinks
\usepackage{url}            % simple URL typesetting
\usepackage{booktabs}       % professional-quality tables
\usepackage{amsfonts}       % blackboard math symbols
\usepackage{nicefrac}       % compact symbols for 1/2, etc.
\usepackage{microtype}      % microtypography

\usepackage{multicol}
\usepackage{amsmath,graphicx}
\usepackage{amssymb}
\usepackage{amsfonts}
\usepackage{xcolor}

\usepackage{amsthm}
\usepackage{mathtools}
\usepackage{multirow}

\usepackage{thm-restate}
\usepackage{textcomp}
\DeclareMathOperator{\conv}{Conv}
\DeclareMathOperator{\tovec}{Vec}
\DeclareMathOperator{\Orth}{O}
\DeclareMathOperator{\SO}{SO}
\DeclareMathOperator{\slice}{split}
\DeclareMathOperator{\concat}{concat}
\def\so{\mathfrak{so}}
\def\L{\mathcal{L}}
\newcommand{\diff}[2]{\frac{\text{d} #1}{\text{d} #2}}

\newtheorem{theorem}{Theorem}

\newtheorem{corollary}[theorem]{Corollary}
\newtheorem{remark}[theorem]{Remark}
\newtheorem{definition}[theorem]{Definition}
\usepackage{caption}
\usepackage{subcaption}

\title{iUNets: Fully invertible U-Nets with Learnable Up- and Downsampling}

\author{%
  Christian Etmann\\
  DAMTP\\
  University of Cambridge\\
  \texttt{cetmann@damtp.cam.ac.uk} \\
   \And
  Rihuan Ke\\
  DAMTP\\
  University of Cambridge\\
  \texttt{rk621@cam.ac.uk} \\
   \And
  Carola-Bibiane Sch\"{o}nlieb\\
  DAMTP\\
  University of Cambridge\\
  \texttt{cbs31@cam.ac.uk} \\
}

\begin{document}

\maketitle

\begin{abstract}
U-Nets have been established as a standard architecture for image-to-image learning problems such as segmentation and inverse problems in imaging. For large-scale data, as it for example appears in 3D medical imaging, the U-Net however has prohibitive memory requirements. Here, we present a new fully-invertible U-Net-based architecture called the iUNet, which employs novel learnable and invertible up- and downsampling operations, thereby making the use of memory-efficient backpropagation possible. This allows us to train deeper and larger networks in practice, under the same GPU memory restrictions. Due to its invertibility, the iUNet can furthermore be used for constructing normalizing flows.
\end{abstract}

\section{Introduction}
\label{sec:intro}
The U-Net \citep{ronneberger2015u} and its numerous variations have become the standard approach for learned segmentation and other different image-to-image tasks. Their general idea is to downsample and later recombine features (via e.g., channel concatenation) with an upsampled branch, thereby allowing for long-term retention and processing of information at different scales. While originally designed for 2D images, their design principle carries over to 3D \citep{cciccek20163d}, where it has found applications in many medical imaging tasks. In these high-dimensional settings, the lack of memory soon poses a problem, as intermediate activations need to be stored for backpropagation. Besides checkpointing, \emph{invertible neural networks} are one possible solution to these memory bottlenecks. The idea is to construct neural networks from invertible (bijective) layers and to successively \emph{reconstruct} activations from activations of deeper layers \citep{gomez2017reversible}. For fully-invertible architectures, this means that the memory demand is \emph{independent of the depth of the network}.\newline
\emph{Partially reversible U-Nets} \citep{brugger2019partially} already apply this principle to U-Nets for each resolution separately. There, since the downsampling is performed with max pooling and the upsampling is performed with trilinear upsampling (both of which are inherently non-invertible operations),  the down- and upsampled activations still have to be stored. Moreover, for other applications in which \emph{full} invertibility is fundamentally needed (such as in normalizing flows \citep{rezende2015variational}), those cannot be used. In this work, we introduce novel \emph{learnable} up- and downsampling operations, which allow for the construction of a fully invertible U-Net (\emph{iUNet}). We apply this iUNet to a learned 3D post-processing task as well as a volumetric medical segmentation task, and use it to construct a normalizing flow.

\section{Invertible Up- and Downsampling}
In this section, we derive learnable \emph{invertible up- and downsampling} operations.\\
\noindent Purely spatially up- and downsampling operators for image data are inherently non-bijective, as they alter  the dimensionality of their input. Classical methods include up- and downsampling with bilinear or bicubic interpolation as well as nearest-neighbour-methods \citep{bredies2018mathematical}. In neural networks and in particular in U-Net-like architectures, downsampling is usually performed either via max-pooling or with strided convolutions. Upsampling on the other hand is typically done via a strided transposed convolution.\newline
One way of \emph{invertibly} downsampling image data in neural networks is known as \emph{pixel shuffle} or \emph{squeezing} \citep{realnvp}, which rearranges the pixels in a $C\times H \times W$-image to a $4C \times H/2 \times W/2$-image, where $C$, $H$ and $W$ denote the number of channels, height and width respectively. Another classical example of such a transformation is the 2D Haar transform, which is a type of Wavelet transform \citep{mallat1999wavelet}. Here, a filter bank is used to decompose an image into \emph{approximation} and \emph{detail coefficients}. These invertible downsampling methods are depicted in Figures \ref{fig:pixelshuffle} and \ref{fig:haar}. In the context of invertible neural networks, these operations have previously been used \citep{ardizzone2019guided} and \citep{lensink2019fully}, the latter of which also use this for invertible upsampling to achieve an autoencoder-like structure. Inverse pixel shuffling on the other hand exhibit problematic artifacts (Fig. \ref{fig:inverse:pixel_shuffle}) when used for invertible downsampling, unless the input features are very non-diverse. This highlights, that extracted features and upsampling operators need to be tuned to one another in order to guarantee both feature diversity as well as outputs which are not inhibited by artifacts. In the following, we will hence introduce novel \emph{learnable} up- and downsampling operations.\\

\iffalse
\begin{figure}
    \centering
    \includegraphics[width=.3\textwidth]{images/up_and_downsampling/barbara.png}\hspace{1mm}
    \includegraphics[width=.3\textwidth]{images/up_and_downsampling/squeeze.png}\hspace{1mm}
    \includegraphics[width=.3\textwidth]{images/up_and_downsampling/haar.png}
    \caption{A test image, the 'pixel shuffle'-transformed image and the Haar-transformed image. The resulting 4 channels are depicted as tiled images. Note that pixel shuffle extracts similar-looking images to the input image, whereas the Haar transform also extracts edge information.}
    \label{fig:general_downsampling}
\end{figure}
\fi

\begin{figure}
     \centering
     \begin{subfigure}[t]{0.19\textwidth}
         \centering
         \includegraphics[width=\textwidth]{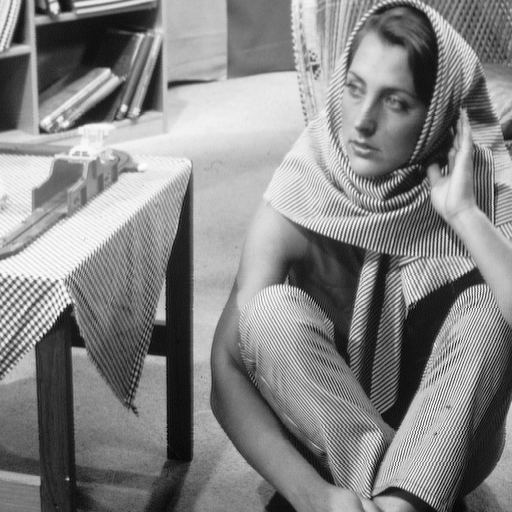}
         \caption{Test image}
         \label{fig:test_image}
     \end{subfigure}
     \hfill
     \begin{subfigure}[t]{0.19\textwidth}
         \centering    \includegraphics[width=\textwidth]{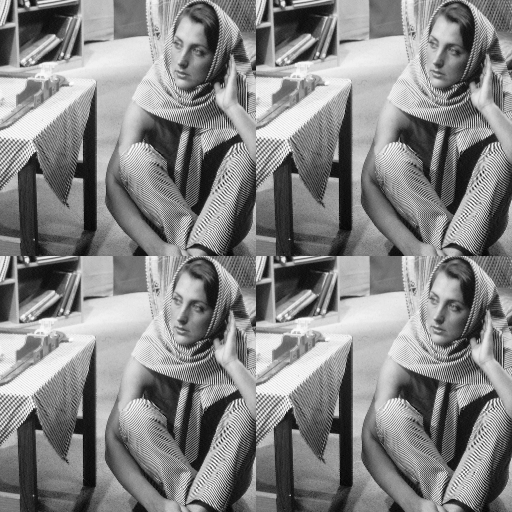}
         \caption{Pixel Shuffle}
         \label{fig:pixelshuffle}
     \end{subfigure}
     \hfill
     \begin{subfigure}[t]{0.19\textwidth}
         \centering
         \includegraphics[width=\textwidth]{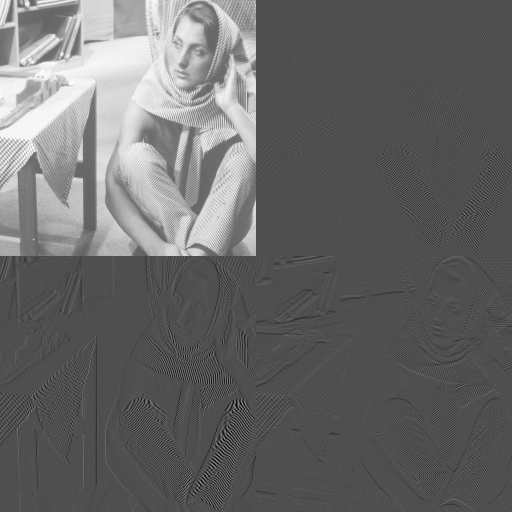}
         \caption{Haar Transform}
         \label{fig:haar}
     \end{subfigure}
     \hfill
     \begin{subfigure}[t]{0.19\textwidth}
         \centering
         \includegraphics[width=\textwidth]{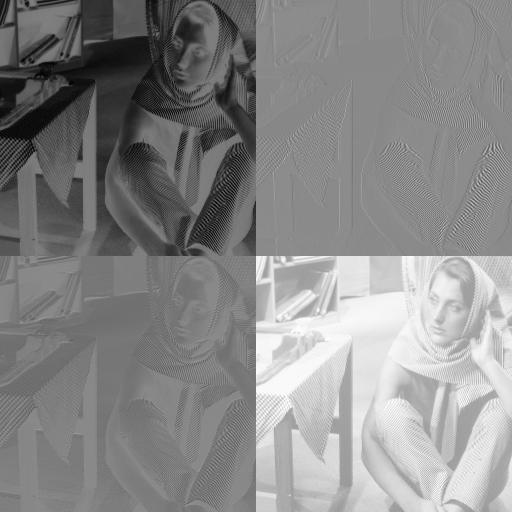}
         \caption{Learnable (random initialization)}
         \label{fig:learned_random}
     \end{subfigure}
     \hfill
     \begin{subfigure}[t]{0.19\textwidth}
         \centering
         \includegraphics[width=\textwidth]{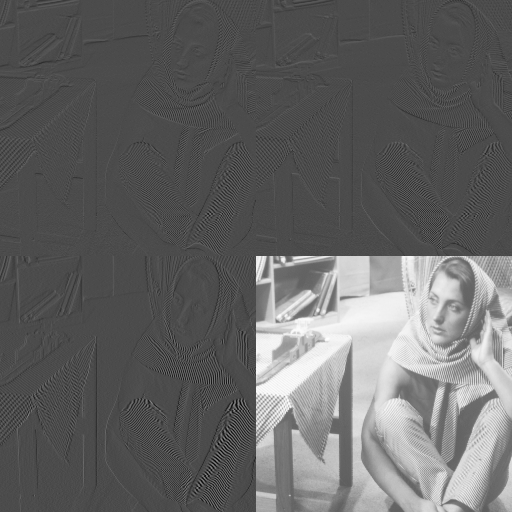}
         \caption{Learnable (minimized $\ell_1$ norm)}
         \label{fig:learned_sparsity}
     \end{subfigure}
\caption{A test image (a), downsampled invertibly using different methods ((b) -- (e)).}
    \label{fig:general_downsampling}
\end{figure}

\noindent The above principle of increasing the number of channels at the same time as decreasing the spatial resolution of each channel guides the creation of these learnable invertible downsampling operators. In the following, we call $d\in \mathbb{N}$ the spatial dimensionality. We say $N\in \mathbb{N}^d$ is divisible by $s\in \mathbb{N}^d$, if $N_i$ is divisible by $s_i$ for all $i\in [d]:=\{1,\dots,d\}$. We denote by $N\oslash s$ the element-wise (Hadamard) division of $N$ by $s$.

\begin{definition}
Let $N \in \mathbb{N}^d$ and the \emph{stride} $s \in \mathbb{N}^d$ for the \emph{spatial dimensionality} $d\in \mathbb{N}$, such that $N$ is divisible by $s$. We call $\sigma:=s_1\cdots s_d$ the \emph{channel multiplier}. For $\tilde{N}:=N \oslash s$ and $\tilde{C}=C \cdot \sigma$, we call 
$$D: \mathbb{R}^{C \times {N}_1 \times \cdots \times {N}_d} \to \mathbb{R}^{\tilde{C}\times \tilde{N}_1 \times \cdots \times \tilde{N}_d}$$
an \emph{invertible downsampling operator} if $D$ is bijective. If the function $D$ is parametrized by $\theta$, i.e. $D=D_\theta$, and $D_\theta$ is invertible for all $\theta \in \mathcal{P}$ (for some parameter space $\mathcal{P}$), then $D_\theta$ is called a \emph{learnable invertible downsampling operator}.
\end{definition}

\begin{remark}\label{rem:2d3dsigma}
For the practically relevant case of stride $2$ in all spatial directions, one has $\sigma=2^2=4$ for 2D data and $\sigma=2^3=8$ for 3D data.
\end{remark}

\noindent Opposite to the downsampling case, in the upsampling case the number of channels needs to be decreased as the spatial resolution of each channel is increased. Using their invertibility, we simply define invertible upsampling operators as inverse downsampling operators.

\begin{definition}
A bijective Operator $U$ is called an \emph{invertible upsampling operator}, if its inverse $U^{-1}$ is an invertible downsampling operator. If the inverse of an operator $U=U_\theta$ is a learnable invertible downsampling operator (parametrized by $\theta \in \mathcal{P}$), then $U_\theta$ is called a \emph{learnable invertible upsampling operator}.
\end{definition}

\begin{figure}
    \centering
    \includegraphics[scale=.7]{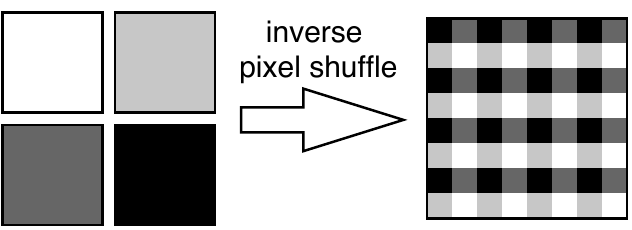}
    \caption{Using the inverse of the pixel shuffling operation will result in checkerboard artifacts.}
    \label{fig:inverse:pixel_shuffle}
\end{figure}

\iffalse
\noindent In Figure \ref{fig:inverse:pixel_shuffle}, the inverse of the \emph{pixel shuffling} is exemplified on monochromatic input channels, which yields checkerboard artefacts (which would additionally result in Moir\'{e} patterns when applying the commonly-used 3-by-3 convolutional kernels on such upsampled features). Thus, when using this type of invertible upsampling, the output will tend to exhibit artefacts, unless the input features are very non-diverse. This highlights, that extracted features and upsampling operators need to be tuned to one another in order to guarantee both feature diversity as well as outputs which are not inhibited by artifacts. In the following, we will present a method to \emph{learn} the appropriate up- and downsampling operation.
\fi

\noindent The general idea of our proposed \emph{learnable} invertible downsampling is to construct a suitable strided convolution operator $D_\theta$ (resulting in a spatial downsampling), which is \emph{orthogonal} (and hence due to the finite dimensionality of the involved spaces, bijective). Its inverse $D_\theta^{-1}$ is thus simply the adjoint operator. Let $D_\theta^\ast$ denote the adjoint operator of $D_\theta$, i.e. the unique linear operator such that $$\langle D_\theta \cdot x, y \rangle_{\mathbb{R}^{\tilde{C}\times \tilde{N}_1 \times \cdots \times \tilde{N}_d}} = \langle x, D^\ast_\theta \cdot y \rangle_{\mathbb{R}^{C\times N_1 \times \cdots \times N_d}}$$ for all $x$, $y$ from the respective spaces, where the $\langle \cdot, \cdot \rangle$ denotes the standard inner products. In this case, $D^\ast_\theta$ is the corresponding \emph{transposed convolution} operator. Hence, once we know how to construct a learnable orthogonal (i.e. invertible) downsampling operator, we know how to calculate its inverse, which is at the same time a learnable orthogonal upsampling operator.

\subsection{Orthogonal Up- and Downsampling Operators as Convolutions}
We will first develop learnable orthogonal downsampling operators for the case $C=1$, which is then generalized. The overall idea is to create an orthogonal matrix and reorder it into a convolutional kernel, with which a correctly strided convolution is an orthogonal operator.\\
\noindent Let $\conv_s (K,x)$ denote the convolution of $x\in \mathbb{R}^{C \times {N}_1 \times \dots \times {N}_d}$ with kernel $K\in \mathbb{R}^{\tilde{C} \times C \times k_1 \times \dots \times k_d}$ and stride $s$, where $k\in \mathbb{N}^d$. 
Corresponding to $s$ the channel multiplier is $\sigma:=s_1\cdots s_d$.
Let further $\Orth(\sigma,\mathbb{R})$ and $\SO (\sigma, \mathbb{R})$ denote the orthogonal and special orthogonal group of real $\sigma$-by-$\sigma$ matrices, respectively. The proofs for this section are contained in Appendix \ref{appdx:main_results_proofs}.

\begin{restatable}{theorem}{orthoConv}\label{thm:orthogonal_conv}
Let $C=1$ and $k=s$. Let further $R: \mathbb{R}^{\sigma \times \sigma} \to \mathbb{R}^{\sigma \times 1 \times s_1 \times \cdots \times s_d}$ be an operator that reorders the entries of a matrix $A \in \mathbb{R}^{\sigma \times \sigma}$, such that the entries of $(RA)_{i,1,\dots} \in \mathbb{R}^{s_1\times \cdots \times s_d}$ consist of the entries of the $i$-th row of $A$. Then for any $A \in \Orth(\sigma,\mathbb{R})$, the strided convolution $$\conv_s (RA, \cdot) :  \mathbb{R}^{1\times N_1 \times \cdots \times N_d} \to \mathbb{R}^{\sigma \times \tilde{N}_1 \times \dots \times \tilde{N}_d}$$ is an invertible downsampling operator. Its inverse is the corresponding transposed convolution.
\end{restatable}

\noindent Note that the assumption that $k=s$ (the strides match the kernel size) will hold for all invertible up- and downsampling operators in the following.
\subsection{Designing Learnable Orthogonal Downsampling Operations}
The above invertible downsampling operator is parametrized over the group $\Orth(\sigma,\mathbb{R})$ of real orthogonal matrices. Note that since orthogonal matrices have $\det = \pm 1$ (i.e. there are two connected components of $\Orth(\sigma,\mathbb{R})$), there is no way to smoothly parametrize the whole parameter set $\Orth(\sigma,\mathbb{R})$. However, if $A \in \SO (\sigma,\mathbb{R}) \subset \Orth(\sigma,\mathbb{R})$, then by switching two rows of $A$, the resulting matrix $A'$ has $\det(A')=-1$. Switching two rows of $A$ simply results in a different order of filters in the kernel $RA$. The resulting downsampling with kernel $RA'$ is thus the same as with kernel $RA$, up to the ordering of feature maps. Hence, the inability to parametrize both connected components of $\Orth(\sigma,\mathbb{R})$ poses no practical limitation, if one can parametrize $\SO (\sigma,\mathbb{R})$. Any such parametrization should be robust and straightforward to compute, as well as differentiable. One such parametrization is the exponentiation of skew-symmetric matrices (i.e. square matrices $S$, for which $S^T=-S$ holds).\\

\noindent From Lie theory \citep{sepanski2007compact}, it is known that the matrix exponential
\begin{equation}
    \exp: \so (\sigma,\mathbb{R}) \to \SO (\sigma, \mathbb{R}) 
\end{equation}
from the Lie algebra $\so (\sigma,\mathbb{R})$ of real skew-symmetric matrices to the Lie group $\SO (\sigma, \mathbb{R})$ is a surjective map (which is true for all compact, connected Lie groups and their respective Lie algebras). This means that one can create \emph{any} special orthogonal matrix by exponentiating a skew-symmetric matrix. The $\sigma$-by-$\sigma$ skew-symmetric matrices can simply be parametrized by 
\begin{equation}
    \theta - \theta^T \in \so (\sigma,\mathbb{R}),
\end{equation}
where $\theta \in \mathbb{R}^{\sigma \times \sigma}$ is a matrix. Note that this is an overparametrization -- any two matrices that differ up to an additive symmetric matrix will yield the same skew-symmetric matrix. Thus, by reordering $\exp (\theta -\theta^T)$ into a convolutional kernel and convolving it with the appropriate stride defines a learnable invertible downsampling operator (for $C=1$, i.e. one input channel).

\begin{corollary}\label{cor:learnable_downsampling_one_channel}
Let the same setting as in Theorem \ref{thm:orthogonal_conv} hold. Then the operator 
$$D_\theta: \mathbb{R}^{1\times N_1 \times \cdots \times N_d} \to \mathbb{R}^{\sigma \times \tilde{N}_1 \times \dots \times \tilde{N}_d}$$
defined by $$D_\theta: x \mapsto \conv_s (R \cdot \exp(\theta-\theta^T),x)$$
is a learnable invertible downsampling operator, parametrized by $\theta$ over the parameter space $\mathbb{R}^{\sigma \times \sigma}$.
\end{corollary}
\noindent Note that both examples of invertible downsampling from Figure \ref{fig:general_downsampling} can be reproduced with this parametrization (up to the ordering of feature maps), as proved in Appendix \ref{appdx:haar_by_exp}. The whole concept is summarized in Fig. \ref{fig:invertible_downsampling} and exemplified in Fig. \ref{fig:learned_random} and \ref{fig:learned_sparsity}. Our implementation of the matrix exponential and its Fr\'{e}chet derivative required for calculating gradients with respect to $\theta$ are described in Appendix \ref{appdx:implementation_details}.

\begin{figure}[h]\centering
    \includegraphics[width=.4\textwidth]{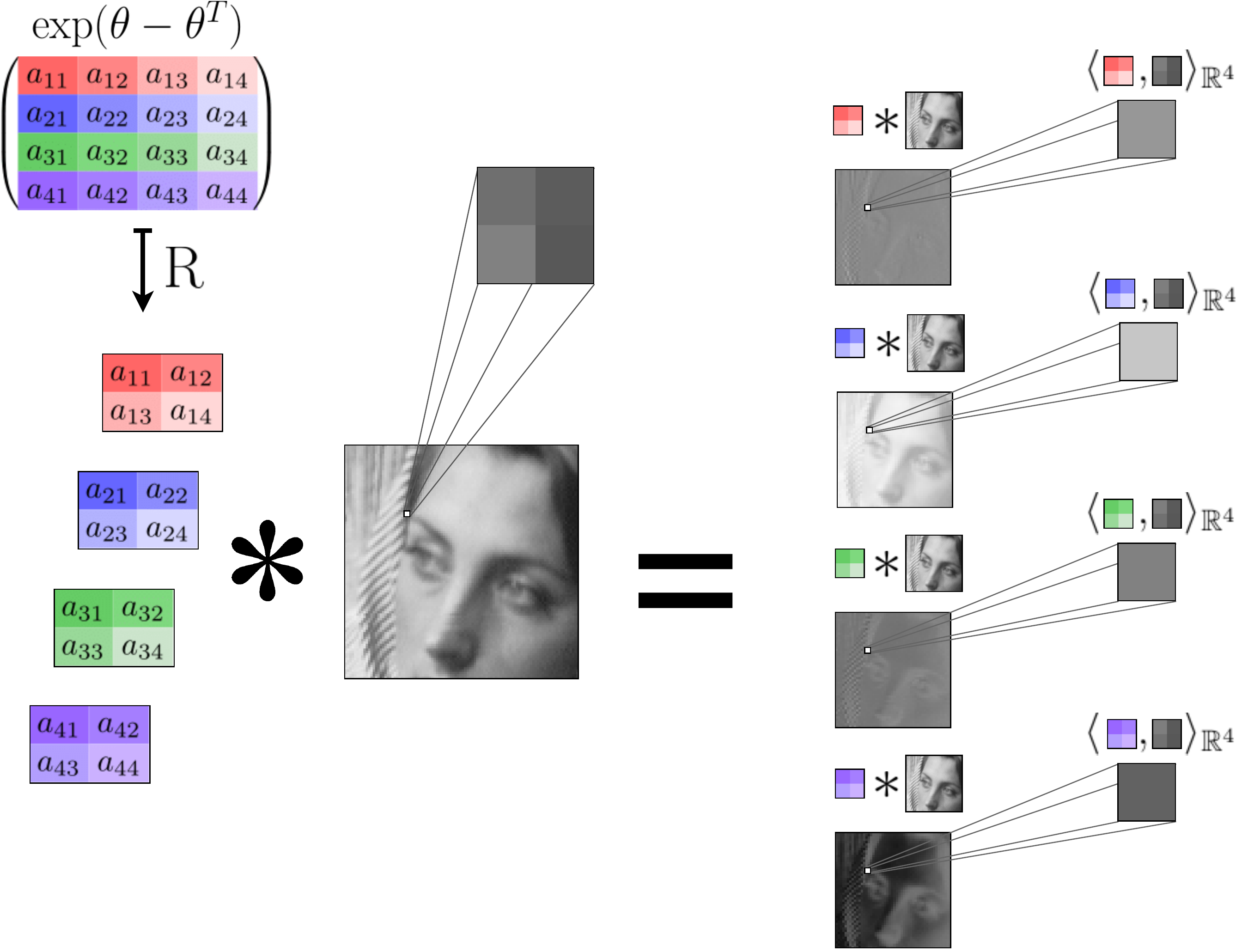}
\caption{Our concept for learnable, invertible downsampling. By exponentiating a skew-symmetric matrix $\theta - \theta^T$, a special orthogonal matrix can be created. When its rows are reordered into filters, convolving these with a stride that matches the kernel size results in an orthogonal convolution, which is a special case of a learnable invertible downsampling operator over the parameter space $\mathbb{R}^{\sigma\times\sigma}$. Because the computational windows of the convolution are non-overlapping, each pixel of the resulting channels is then just the standard inner product of the respective filter with the corresponding image patch in the original image.}\label{fig:invertible_downsampling}
\end{figure}

\iffalse
\begin{figure}
    \centering
    \includegraphics[width=.3\textwidth]{images/up_and_downsampling/barbara.png}\hspace{1mm}
    \includegraphics[width=.3\textwidth]{images/up_and_downsampling/random_init_3.png}\hspace{1mm}
    \includegraphics[width=.3\textwidth]{images/up_and_downsampling/learned_general_sparsity.png}
    \caption{A test image, a randomly initialized learnable invertible downsampling and a learnable downsampling trained to create a sparse representation (by minimizing the 1-norm of the representation). The latter looks very similar to a Haar decomposition.}
    \label{fig:learned_example}
\end{figure}
\fi

\noindent The case of $C=1$ can now be easily generalized to an arbitrary number of channels $C$ by applying the learnable invertible downsampling operation to each input channel independently (precise statement in Appendix \ref{appdx:main_results_proofs}).\\

Aside from exponentiating skew-symmetric matrices, orthogonal matrices can be obtained by Cayley transforms, products of Householder matrices as well as products of Givens rotations, some of which have previously been explored in the invertible neural networks literature for obtaining easily invertible layers, e.g. by parametrizing $1\times 1$-convolutions with these matrices \citep{ardizzone2019guided}\citep{reparamlie}\citep{golinski2019improving}\citep{inverttolearntoinvert}\citep{tomczak2016improving}.

\section{Invertible U-Nets}
The general principle of the classic U-Net \citep{ronneberger2015u} is to calculate features on multiple scales by a sequence of convolutional layers and downsampling operations in conjunction with an increase in the number of feature maps. The downsampled feature maps tend to capture large-scale features, whereas the more highly resolved feature maps capture more fine-grained properties of the data. The low-resolution features are successively recombined with the prior, high-resolution features via feature map concatenation, until the original spatial resolution of the input data is reached again.\\

\noindent In order to construct a fully invertible U-Net (\emph{iUNet}), we adopt these same principles. A depiction of the iUNet is found in Figure \ref{fig:iunet}. Note that unlike in the case of non-invertible networks, the total data dimensionality may not change -- in particular the number of channels may not change if the spatial dimensions remain the same. \newline
Unlike in the case of the classic U-Net, not all feature maps of a certain resolution can be concatenated with the later upsampled branch, as this would violate the condition of constant dimensionality. Instead, we split the $C$ feature maps into two portions of $\lambda C$ and $(1-\lambda)C$ channels (for appropriate \emph{split fraction} $\lambda$, s.t. $C > \lambda C \in \mathbb{N}$). The portion with $\lambda C$ channels gets processed further (cf. the gray blocks in Figure \ref{fig:iunet}), whereas the other portion is later concatenated with the upsampling branch (cf. the green blocks in Figure \ref{fig:iunet}). Splitting and concatenating feature maps are invertible operations.\newline
\noindent While in the classic U-Net, increasing the number of channels and spatial downsampling via max-pooling are separate operations, these need to be inherently linked for invertibility. This is achieved through our \emph{learnable invertible downsampling} to the (non-concatenated) split portion.

\iffalse
\noindent Mathematically, we define the left side of the invertible U-Net for each scale $i=1,\dots,m$ via
\begin{equation}
    \begin{aligned}
        y^L_i&=\Phi^L_i(x^L_i)& \\
        (\tilde{y}^L_i, c_i)&= \slice_i (y^L_i) &\text{ if } i<m \phantom{,} \\
        x^L_{i+1}&= D_i(\tilde{y}^L_i) &\text{ if } i<m,
    \end{aligned}\label{eq:iunet_left}
\end{equation}
where $x^L_1$ is the network's input. Here, $\Phi^L_i$, $\slice_i$ and $D_i$ denote appropriate invertible sub-networks, channel splitting and invertible downsampling operations. With $x^R_m=y^L_m$, the right side is defined via
\begin{equation}
    \begin{aligned}
        y^R_i&=\Phi^R_i(x^R_i) &\\
        \tilde{y}^R_{i-1}&= U_i(y^R_i) &\text{ if } i>1\phantom{,}\\
        x^R_{i-1} &= \concat_i (\tilde{y}^R_{i-1},c_{i-1})&\text{ if } i>1,
    \end{aligned}\label{eq:iunet_right}
\end{equation}
for $i=m,\dots,1$. Again, $\Phi^R_i$, $\concat_i$ and $U_i$ denote the respective invertible sub-networks, channel concatenation and invertible upsampling. The iUNet is then defined to be the function that maps $x^L_1$ to $y^R_1$, i.e. $f: x^L_1 \mapsto y^R_1$. 
\fi

Mathematically, for scale $i\in [m]$, let $\Phi_i^L$ and $\Phi_i^R$ denote functions defined as sequences of invertible layers. For $i<m$, let $D_i$ and $U_i$ denote the invertible down- respectively upsampling operators and let $\text{split}_i$ and $\text{concat}_i$ denote the channel splitting and concatenation operators. For input $x^1_L$,
\begin{equation}
\begin{aligned}[c]
        y^L_i&=\Phi^L_i(x^L_i)& \\
        (\tilde{y}^L_i, c_i)&= \slice_i (y^L_i) &\text{ if } i<m  \\
        x^L_{i+1}&= D_i(\tilde{y}^L_i) &\text{ if } i<m
\end{aligned}
\qquad \qquad \qquad
\begin{aligned}[c]
        y^R_i&=\Phi^R_i(x^R_i) &\\
        \tilde{y}^R_{i-1}&= U_{i-1}(y^R_i) &\text{ if } i>1\\
        x^R_{i-1} &= \concat_{i-1} (\tilde{y}^R_{i-1},c_{i-1})&\text{ if } i>1
\end{aligned} \label{eq:iunet_definition}
\end{equation}
defines a function $f: x^L_1 \mapsto y^L_i$, which is our iUNet (see Figure \ref{fig:iunet}).

%\begin{equation*}
%\begin{array}{rlcrl}
%        y^L_i\hspace{-3mm}&=\Phi^L_i(x^L_i) & & y^R_i\hspace{-3mm}&=\Phi^R_i(x^L_i)\\
%        (a_i, b_i)\hspace{-3mm}&= \slice_i (y^L_i) & & a^R_{i-1}\hspace{-3mm}&= U_i(y_i)\\
%        x^L_{i+1}\hspace{-3mm}&= D_i(a^L_i) & & x^R_{i-1}\hspace{-3mm}&= \concat_i (a^R_{i-1},b^{i-1})
%\end{array}
%\end{equation*}
%\CE{blabla}

\begin{remark}\label{rem:channel_exponential_increase}
The number of channels increases exponentially as the spatial resolution decreases. The base of the exponentiation not only depends on the channel multiplier $\sigma$, but also on the channel split fraction $\lambda$, since only this fraction of channels gets invertibly downsampled. The number of channels thus increases by a factor of $\lambda \cdot \sigma$ between two resolutions. E.g. in 2D, $\lambda=1/2$ leads to a doubling of channels for $s=(2,2)$, whereas in 3D the split fraction $\lambda=1/4 $ is required to achieve a doubling of channels for $s=(2,2,2)$. This fine-grained control is in contrast to \citep{lensink2019fully}, where in 2D, the number of channels is always multiplied by 4 (in 2D) respectively 8 (in 3D), which makes a large number of downsampling operations infeasible.
\end{remark}

\subsection{Application of invertible U-Nets}
The above discussion already pointed towards some commonalities and differences between the iUNet and non-invertible U-Nets. The restriction that the dimensionality may not change between each invertible sub-network's input and output imposes constraints both on the architecture as well as the data. When invertibly downsampling, the spatial dimensions need to be exactly divisble by the strides of the downsampling. This is in contrast to non-invertible downsampling, where e.g. padding or cropping can be introduced. Furthermore, due to the application of channel splitting (or if one employs coupling layers), the number of channels needs to be at least 2. An alternative may be exchanging the order of invertible downsampling and channel splitting. \\
\noindent These restrictions may prove to be too strong in practice for reaching a certain performance for tasks in which \emph{full invertibility} is not strictly needed. For this, the number of channels before and after the iUNet can be changed (e.g. via a convolution), such that the memory-efficient backpropagation procedure can still be applied to the whole \emph{fully invertible sub-network}, i.e. the whole iUNet (see Fig. \ref{fig:iunet}. A general issue in memory-efficient backpropagation is stability of the inversion \citep{behrmann2020on}, which in turn influences the stability of the training. We found that using group or layer normalization \citep{groupnorm,ba2016layer} were effective means of stabilizing the iUNet in practice. 

%This means that if one uses e.g. a (linear) convolutional layer before the iUNet to increase the number of channels, the memory-efficient backpropagation procedure can still be applied to the whole \emph{fully invertible sub-network}, i.e. the whole iUNet. Similarly, a linear convolutional layer can be applied to the output of the iUNet in order to change the number of channels to some desired number, e.g. the number of classes for a semantic segmentation problem. Note that adding a (learnable) linear layer each before as well as after the invertible network does not necessitate storing any additional activations (since derivatives of linear maps are simply the linear maps themselves).

\begin{figure*}[h]
\includegraphics[width=\textwidth]{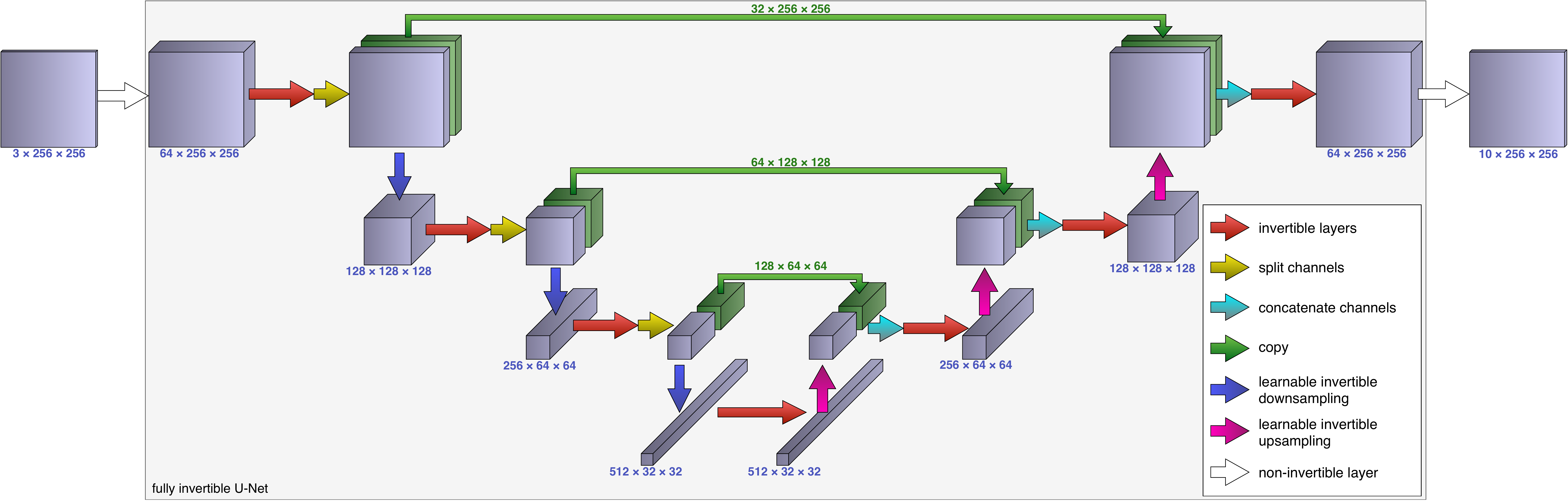}
\caption{Example of a 2D iUNet used for memory-efficient backpropagation for segmenting RGB-images into 10 classes. Linear convolutions are used to increase the number of channels to a desired number (64 in this example), which then determines the input and output data dimensionality of the invertible U-Net. Invertible layers, invertible up- and downsampling and skip connections in conjunction with channel splitting and concatenation make up the invertible U-Net (contained in the light-blue box).}\label{fig:iunet}
\end{figure*}

\subsection{Normalizing Flows}
Normalizing flows \citep{rezende2015variational} are a class of generative models, which -- much like generative adversarial networks \citep{goodfellow2014generative} and variational autoencoders \citep{variationalautoencoder} -- learn to map points from a simple, known distribution $q$ to points from a more complicated distribution $p$. Unlike these models, the use of a (locally diffeomorphic) invertible neural network $f$ allows for the evaluation of the likelihood of points under this model by employing the change-of-variables formula for probability densities. Let $z\sim q(z)$ (e.g. a normal distribution) and $x:=f^{-1}(z) \sim p(x)$, then
\begin{equation}
    \log p(x) = \log q(f(x)) + \log \left| \det \frac{\text{d} f(x)}{\text{d} x}  \right| \label{eq:log_likelihood_normalising}
\end{equation}
is the log-likelihood of $x$ under this model. For training $f$ as a maximum likelihood estimator over a training set, one thus needs to be able to evaluate the determinant-term in eq. \eqref{eq:log_likelihood_normalising}. Depending on the specific invertible layer, different strategies for evaluating this term exist, see e.g., \citep{dinh2016density} and \citep{resflows}.
If $f$ is an iUNet (according to the definitions in \eqref{eq:iunet_definition}), then $$\det \frac{\text{d} f(x_1^L)}{\text{d} x_1^L} = \prod\limits_{i=1}^{m} \det \frac{\text{d} \Phi^L_i(x^L_i)}{\text{d} x^L_i} \cdot \det \frac{\text{d} \Phi^R_i(x^R_i)}{\text{d} x^R_i},$$
indicating that only the nonlinear, invertible layers (e.g. coupling blocks) contribute to the 'log-abs-det' term.
This is because all up- and downsampling operations as well as the channel splitting and concatenation operations are special orthogonal, yielding unit determinants. The detailed statement and proof are found in Appendix \ref{appdx:normalizing_flows}. The iUNet can be regarded as an alternative to the 'factoring out' approach from \citep{dinh2016density}, which can roughly be understood as a multi-scale approach in which coarse-scale (more strongly downsampled) features are not fed back into the fine-grained feature extractors. We hence hope that the iUNet approach is more suited to normalizing flows, much like U-Net-like architectures have proven to work well with multi-scale features for e.g., segmentation tasks.

\section{Experiments}
In the following, results of an experiment on learned 3D post-processing from imperfect CT reconstruction, as well as a 3D segmentation experiment are presented. In all trained iUNets, the \emph{additive coupling layers} as defined in \citep{jacobsen2018revnet} were used. The models were implemented in \texttt{Pytorch} using the library \texttt{MemCNN} \citep{vandeLeemput2019MemCNN} for memory-efficient backpropagation. Moreover, we demonstrate the capability of iUNets as normalizing flows. In Appendix \ref{appdx:runtimes}, we additionally compare the differing runtimes between conventional training and memory-efficient training.

\subsection{Learned Post-Processing of Imperfect 3D CT Reconstructions}
The goal of this experiment is to test the invertible U-Net in a challenging, high-dimensional learned post-processing task on imperfect 3D CT reconstructions, where the induced undersampling artifacts appear on a large, three-dimensional scale. 
For this experiment, we created an artificial dataset of undersampled, low-dose CT reconstructions of the 3D 'foam phantoms' from \citep{pelt2018improving} at a resolution of $256^3$. The initial reconstruction was performed using filtered backprojection (FBP). In Appendix \ref{appdx:postprocessing_details}, a detailed description of the experimental setup is provided.\\
\noindent As indicated in Table \ref{tbl:foam_results}, even the worst-performing iUNet performed considerably better than the best-performing classic U-Net, both in terms of PSNR but especially in terms of SSIM. The U-Net contained an additive skip connection from the input to the output, which considerably improved its performance. While both model classes benefitted from an increased channel blowup, only the invertible U-Net benefits from raising the number of scales from 4 to 8 (at which point the receptive field spans the whole volume). The fact that the classic U-Net drops in performance despite a higher model capacity may indicate that the optimization is more problematic in this case. The invertible U-Net shows one of its advantages in this application: By initializing the layer normalization as the zero-mapping in each coupling block, the whole iUNet was initialized as the identity function. At initialization, \emph{each} convolutional layer's input is thus a part of the whole model input (up to an orthogonal transform). Since the optimal function for learned post-processing can be expected to be close to the identity function, we assume that this initialization is well-suited for this task. We further used the memory-efficiency of the invertible U-Net to double the channel blowup compared to the largest classic U-Net that we were able to fit into memory. This brought further performance improvements, showing that a higher model capacity can aid in such tasks.\newline 
In Figure \ref{fig:foam_example}, a test sample processed by the best-perfoming classic iUNet as well as classic U-Net are shown, along with the ground truth and the FBP reconstruction. Apart from the overall lower noise level, the iUNet is able to discern neighbouring holes from one another much better than the classic U-Net. Moreover, in this example a hole that is occluded by noise in the FBP reconstruction does not seem to be recognized as such by the classic U-Net, but is well-differentiated by the iUNet.

\begin{figure}[h]
    \centering
       \includegraphics[width=.22\textwidth]{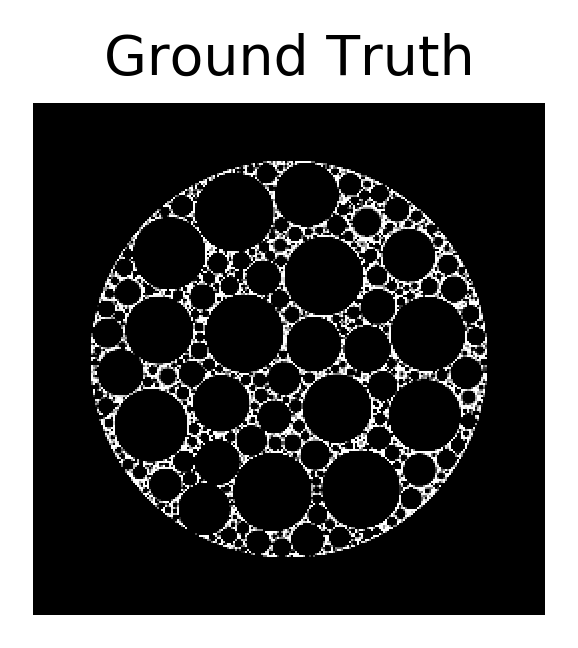}   \includegraphics[width=.22\textwidth]{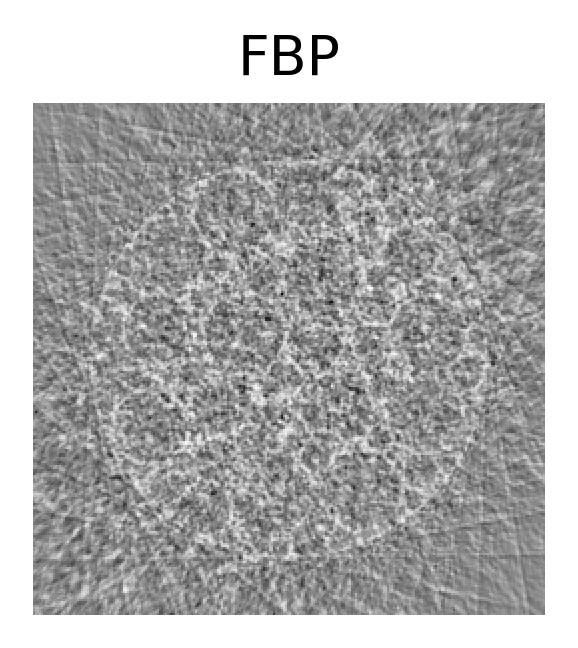}
       \includegraphics[width=.22\textwidth]{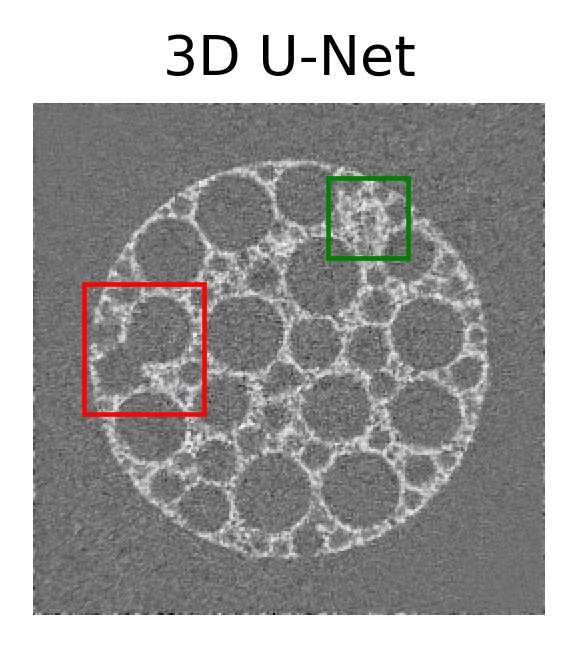}  \includegraphics[width=.22\textwidth]{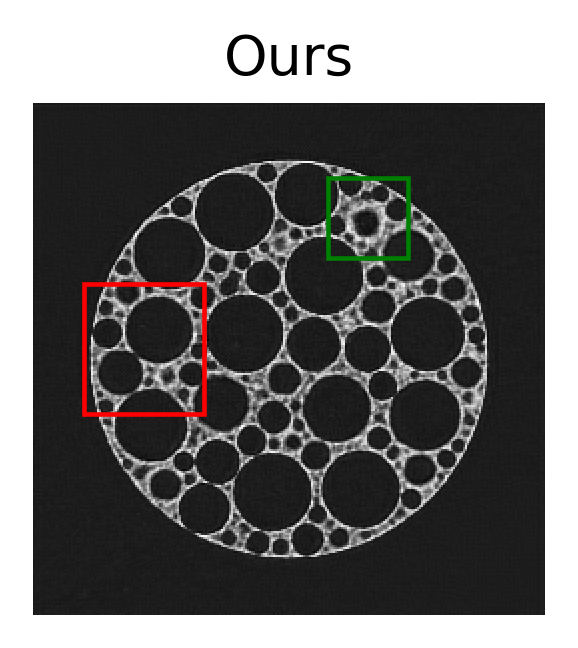}
    \caption{Slices through 3D volume from test set (post-processing task). Apart from the lower noise level compared to the classic 3D U-Net, the 3D iUNet is also able to differentiate much better between neighbouring holes (red) and discerning holes from noise (green).}
    \label{fig:foam_example}
\end{figure}

\begin{table}%[hbr]
\centering
\caption{Results of learned post-processing experiments. Here, 'scales' indicates the number of different resolutions, whereas 'channel blowup' denotes the number of feature maps before reverting to one feature map again.}
\label{tbl:foam_results}
\begin{tabular}{llllll}
\toprule
\multirow{2}{*}{scales} & \multirow{2}{*}{\shortstack[l]{channel\\ blowup}} & \multicolumn{2}{c}{3D U-Net} & \multicolumn{2}{c}{3D iUNet} \\
&  & SSIM & PSNR & SSIM &PSNR \\
\midrule
$4$ & $4$ & 0.302 & 13.29 & 0.568 & 14.00 \\
$4$ & $8$ & 0.416 & 13.89 & 0.780 & 14.99 \\
$8$ & $4$ & 0.236 & 12.42 & 0.768 & 15.10 \\
$8$ & $8$ & 0.425 & 13.92 & 0.829 & 15.82 \\
$8$ & $16$ & \phantom{---}- & \phantom{---}- & \textbf{0.854} & \textbf{16.11} \\
\bottomrule
\end{tabular}
\end{table}

\subsection{Brain Tumor Segmentation}

%In this experiment, we study the performance of the iUNet for volumetric brain tumor segmentation. The experiment is carried out on a benchmark brain tumor dataset from the  BraTS 2018 challenge  \citep{menze2014multimodal}. The training dataset for BraTS 2018 includes 285 multi-parametric MRI scans with ground truth labels collected by expert neuroradiologists. We further split the 285 scans into 91\% and 9\% for training and validation respectively. 

The following experiment is based on the multi-parametric MRI brain images from the brain tumor segmentation benchmark BraTS 2018 \citep{menze2014multimodal}, which is a challenging task due to the involved high dimensionality of the 3D volumes.

%\noindent For this dataset, we consider three different sizes of invertible networks, with a channel blowup of 16, 32, and 64 channels respectively. For comparison, we also train a baseline 3D U-Net \citep{cciccek20163d}. In our implementation, we have $5$ different levels of resolutions in both the U-Net and the invertible networks, starting from a cropping size of $160 \times 192 \times 128$ and $4$ input channels (corresponding to $4$ modalities). For the baseline U-Net, the input is followed by a blowup to $24$, and is then doubled after each down-sampling. The invertible networks have a channel split of $1/4$ after the invertible $3$D down-sampling, meaning that the channel numbers are doubled as well. After each channel split, two additive coupling layers are used.   The results on the BraTS validation set (including 66 scans), measured in the dice score and sensitivity \citep{bakas2018identifying} are reported in Table \ref{tab:BraTS2018}.  According to the table, the increases of the channel numbers in the invertible networks lead to a gain in the performance in terms of both accuracy and the sensitivity. The largest invertible network outperforms the baseline U-Net, thanks to the memory saving and bigger channel numbers under similar hardware configurations. 

\noindent Here, we split the BraTS 2018 training set (including 285 multi-parametric MRI scans) into $91\%$ for training and $9\%$ for validation. Three types of tumor sub-regions, namely enhancing tumor (ET), whole tumor (WT) and tumor core (TC), are segmented and evaluated. The networks are trained on the $91\%$ annotated data, where the ground truth labels were collected by expert neuroradiologists. 
For this dataset, we consider three different sizes of invertible networks, with a channel blowup of 16, 32, and 64 channels respectively. For comparison, we also train a baseline 3D U-Net \citep{cciccek20163d}. In our implementation, we use $5$ different levels of resolutions for both the U-Net and the invertible networks, starting from a cropping size of $160 \times 192 \times 128$ and $4$ input channels corresponding to the MRI modalities (T1, T1-weighted, T2-weighted and FLAIR). Each $\Psi_i^L$ and $\Phi^R_i$ was parametrized by two additive coupling layers. Group normalization (with group size $8$) was used both for the U-Net and iUNet.  For the baseline U-Net, the input is followed by a blowup to $24$ channels, which is then doubled after each downsampling. This is the largest number of channels that we were able to fit into GPU memory in our experiments.
The invertible networks employ a channel split of $1/4$ (Remark \ref{rem:channel_exponential_increase}), meaning that the number of channels is doubled when the spatial resolution is decreased. 
%After each channel split, two additive coupling layers are used. 

In all cases, a final convolutional layer with a sigmoid nonlinearity maps the iUNet's output feature maps to the three channels associated with ET, WT and TC sub-regions respectively. The used training loss function is the averaged Dice loss for the ET, WT and TC regions respectively. 

In Table \ref{tab:BraTS2018}, we report the results on the BraTS validation set (including 66 scans), measured in terms of Dice score and sensitivity \citep{bakas2018identifying} respectively. According to the table, the increases of the channel numbers in the invertible networks lead to a gain in the performance in terms of the Dice score as well as
the sensitive. 
Thanks to the memory-efficiency and thus a larger possible number of channels under similar hardware configurations, iUNets that were larger than the baseline U-Net outperform this baseline U-Net, profiting from a larger model capacity. In Appendix \ref{appdx:segmentation_details}, we further demonstrate that (for the smallest iUNet) the memory-efficient and conventional backpropagation lead to comparable loss curves.

\begin{table}[ht]
    \centering
    \caption{Results on BraTS2018 validation set.}
    \label{tab:BraTS2018}
    \bgroup
    \setlength{\tabcolsep}{4pt}
    \begin{tabular}{lccccccccc}
    \hline
     & \multicolumn{4}{c}{\phantom{.}Dice score} && \multicolumn{4}{c}{Sensitivity} \\
      & ET & WT & TC & avg & & ET & WT & TC & avg \\
    \hline
     U-Net & 0.770 & \textbf{0.901} & 0.828 & 0.833 & \phantom{W} 
     %& 2.97 & 6.18 & 8.57 & 5.91 
     & 0.776 & 0.914 & 0.813 & 0.834
     \\
    \hline
      iUNet-16 & 0.767 & 0.900 & 0.809 & 0.825 
    %&& 3.77 & 5.27 & 8.66 & 5.90
    && 0.779 & 0.916 & 0.798 & 0.831
    \\
    iUNet-32 & 0.782 & 0.899 & 0.825 & 0.835 
    %& & 3.31 & 5.79 & 8.39 & 5.83 
    && 0.773 & 0.908 & 0.824 & 0.835
    \\
    iUNet-64 & \textbf{0.801} & 0.898 & \textbf{0.850} & \textbf{0.850} 
    %& &2.99 &	7.51 &	7.43 & 5.98
    && \textbf{0.796} & \textbf{0.918} & \textbf{0.829} & \textbf{0.848}
    \\
    % \hline
    % \multirow{3}{*}{learned} & \multirow{3}{*}{2} &    iUNet-16 \\   
    % && iUNet-32     & 0.701 & 0.890 & 0.773 &&& 6.038 & 6.946 & 11.215 \\  
    % && iUNet-64 
    % & 0.767 & 0.889 & 0.791 &&& 4.158 & 7.037 & 8.250 \\    
    \hline
    \end{tabular}
    \egroup
\end{table}

\subsection{Normalizing Flows}
In order to show the general feasibility of training iUNets as normalizing flows, we constructed an iUNet with 3 invertible downsampling operations. At each scale, 4 affine coupling layers \citep{dinh2016density} were used. Initially, the data was downsampled with a pixel shuffle operation, which turned the $3\times 32 \times 32$-images into images of size $12 \times 16 \times 16$, which allowed for the use of $\lambda=1/2$ at every scale. We trained the iUNet for 400 epochs on CIFAR10, which yielded a negative log-likelihood (NLL) of 3.60 bits/dim on the test set. This is somewhat worse than the most comparable method (Real NVP) by \citet{dinh2016density}, who report an NLL of 3.49 bits/dim. We suspect that the lower performance may stem from similar effects as reported in \citep{pmlr-v97-behrmann19a}, where inverses were calculated via truncated series. As analyzed in \citep{resflows}, this yields a biased estimator of the likelihood \eqref{eq:log_likelihood_normalising}. The authors propose to use a stochastic truncation \citep{kahn1955use}, which provides an unbiased estimator and improves the measured performance. Since we also use a series truncation for the matrix exponential (see Appendix \ref{appdx:implementation_details}), we conjecture that such an approach may debias our likelihood estimator as well, thereby improving our performance.
Fig. \ref{fig:normalizing_flows_cifar10} depict examples that were randomly generated by our model.
\begin{figure}[h]
    \centering
    \includegraphics[width=.48\textwidth]{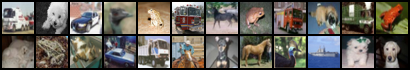}\hspace{.03\textwidth}
    \includegraphics[width=.48\textwidth]{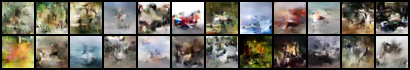}
    \caption{Left: Example images from CIFAR10. Right: Randomly picked images generated by an iUNet trained as a normalizing flow.}
    \label{fig:normalizing_flows_cifar10}
\end{figure}

\section{Conclusion and Future Work}
In this work, we introduced a fully invertible U-Net (iUNet), which employs a novel learnable invertible up- and downsampling. These are orthogonal convolutional operators, whose kernels are created by exponentiating a skew-symmetric matrix and reordering its entries. We show the viability of the iUNet for memory-efficient training on two tasks, 3D learned post-processing for CT reconstructions as well as volumetric segmentation. On both the segmentation as well as the CT post-processing task, the iUNet outperformed its non-invertible counterparts; in the case of the post-processing task even substantially. We therefore conclude that the iUNet should be used e.g. for high-dimensional tasks, in which a classic U-Net is not feasible.\newline
We have further demonstrated the general  feasibility of the iUNet structure for normalizing flows, although in terms of likelihood it performed somewhat worse than the comparison method.\\

In future work, we would therefore like to extend upon this work and find out how to improve the performance of iUNet-based normalizing flows.

\section*{Acknowledgements}
The authors thank Sil van de Leemput for his help in using and extending MemCNN, as well as Jens Behrmann for useful discussions around normalizing flows. CE and CBS acknowledge support from the Wellcome Innovator Award RG98755. RK and CBS acknowledge support from the EPSRC grant EP/T003553/1. CE additionally acknowledges partial funding by the Deutsche Forschungsgemeinschaft (DFG) - Projektnummer 281474342: ’RTG $\pi^3$ - Parameter Identification - Analysis, Algorithms, Applications’ for parts of the work done while being a member of RTG $\pi^3$. CBS additionally acknowledges support from the Leverhulme Trust project on ‘Breaking the non-convexity barrier’, the Philip Leverhulme Prize, the EPSRC grant EP/S026045/1, the EPSRC Centre Nr. EP/N014588/1, the RISE projects CHiPS and NoMADS, the Cantab Capital Institute for the Mathematics of Information and the Alan Turing Institute.

\section*{Broader Impact}
This work allows for the training of neural networks on high-dimensional imaging data. While we envision it as a tool for medical imaging applications, the methods presented here are general and can in principle be applied to all kinds of image data.

\bibliographystyle{plainnat}
\bibliography{refs}

\newpage
\appendix
{\Large Appendix}
\section{Orthogonal Learnable Downsampling}\label{appdx:ortho}
\subsection{Mathematical Preliminaries}
In the following, we introduce a simple notation which allows for a mathematically rigorous treatment of the presented theory.\\

For a tuple of matrices $(M_1,\dots,M_n)$ (where the matrices may be of different sizes), the \emph{direct sum} of these matrices is defined as the block diagonal matrix
\begin{equation}
    \bigoplus\limits_{i=1}^{n} M_i := M_1 \oplus \cdots \oplus M_n := \begin{pmatrix}
             M_1 &  &  \\
              & \hspace{-.3mm} \ddots \hspace{-.3mm}&   \\
             &  & M_n\phantom{.} 
            \end{pmatrix},
\end{equation}
which in particular implies $\det ( \oplus_{i=1}^n M_i) = \prod_{i=1}^n \det (M_i)$. Analogously, for a tuple $(f_1,\dots,f_n)$ of functions $f_i: A_i \to B_i$ (where $A_i$ and $B_i$ are some sets), we write 
\begin{equation}
    \begin{aligned}
        \oplus_{i=1}^{n} f_i : A_1 \times \cdots \times A_n &\to B_1 \times \cdots \times B_n \\
                            (a_1,\dots,a_n)&\mapsto (f_1(a_1),\dots,f_n(a_n)).
    \end{aligned}    
\end{equation}
Note that this construction is sometimes also called the cartesian product of functions, but rarely the direct sum of functions.

If additionally, $A_i=\mathbb{R}^{k_i}$ and $B_i=\mathbb{R}^{l_i}$ for some $k_i,l_i \in \mathbb{N}$ and $f_i: x \mapsto M_i x$, where $M_i \in \mathbb{R}^{l_i \times k_i}$ for all $i\in [m]$, then $\oplus_{i=1}^{n} f_i \cong \oplus_{i=1}^{n} M_i$. More generally, if the $f_i$ map linearly between finite-dimensional real vector spaces, $\oplus_{i=1}^{n} f_i$ is isomorphic to a direct sum of $n$ associated matrices.

Note that in particular, given sufficient differentiability of the $f_i$,
\begin{equation}
    \begin{aligned}
        \frac{\text{d} (\oplus_{i=1}^n f_i) (a_1,\dots,a_n)}{\text{d} (a_1,\dots,a_n)} = \bigoplus_{i=1}^n \frac{\text{d}  f (a_i)}{\text{d} a_i} 
    \end{aligned}
\end{equation}
holds. This is true both if we view the derivative as a linear function (in the sense of e.g. Fr\'{e}chet derivatives) or, up to isomorphism, if we identify this function with its corresponding Jacobian matrix.

\subsection{Main Results} \label{appdx:main_results_proofs}
As in the main paper, we denote by $\conv_s (K,x)$ the convolution of $x\in \mathbb{R}^{C \times {N}_1 \times \dots \times {N}_d}$ with a kernel $K\in \mathbb{R}^{\tilde{C} \times C \times k_1 \times \dots \times k_d}$ and stride $s$, where $k\in \mathbb{N}^d$. In the following, we will assume $N$ to be divisible by $s$ and define $\sigma:=s_1\cdots s_d$, $\tilde{C}:=\sigma C$ and $\tilde{N}:=N \oslash s$.

\orthoConv*
\begin{proof}
Let $X:=\mathbb{R}^{1\times N_1 \times \cdots \times N_d}$ and 
If $k_i=s_i$ for all $i\in [d]$ (i.e. the kernel size matches the strides), then the computational windows of the discrete convolution are non-overlapping. This means that each entry of $y=\conv_s (RA,x)$ is the result of the multiplication of the $\sigma$-by-$\sigma$-matrix $A$ with a $\sigma$-dimensional column vector of the appropriate entries from $x$. This means that
\begin{equation}\label{eq:proof_direct_sum}
    \begin{aligned}
        \tilde{\tovec} \left(\conv_s (RA,x) \right) &= \left( \bigoplus\limits_{j=1}^{\tilde{N}_1 \cdots \tilde{N}_d} A \right) \cdot \tovec (x)
         &=\begingroup 
            \setlength\arraycolsep{2pt} \renewcommand{\arraystretch}{0} \begin{pmatrix}
             A &  &  \\
              & \hspace{-.3mm} \ddots \hspace{-.3mm}&   \\
             &  & A\phantom{.} 
            \end{pmatrix} \endgroup \cdot \tovec(x), \\
    \end{aligned}
\end{equation}
where $\tovec : \mathbb{R}^{1\times N_1 \times \cdots \times N_d} \to \mathbb{R}^{N_1 \cdots N_d}$ and $\tilde{\tovec} : \mathbb{R}^{\sigma \times \tilde{N}_1 \times \cdots \times \tilde{N}_d} \to \mathbb{R}^{\sigma \cdot \tilde{N}_1 \cdots \tilde{N}_d}$ denote appropriate reordering operators (into column vectors). Note that reordering operators are always orthogonal.
We will now show that if the block diagonal matrix $\hat{A}:=\bigoplus_{j=1}^{\prod_i \tilde{N}_i} A$ is orthogonal, then the convolution $\conv_s(RA,\cdot)$ is an orthogonal operator. Since $$ \hat{A}^T \hat{A} = \bigoplus_{j=1}^{\prod_i \tilde{N}_i} A^T A \text{\phantom{.} and \phantom{.}}  \hat{A} \hat{A}^T = \bigoplus_{j=1}^{\prod_i \tilde{N}_i} A A^T,$$
$A$ being orthogonal implies $\hat{A}$ being orthogonal. 
For any $a,b \in \mathbb{R}^{\tilde{C}\times \tilde{N}_1 \times \cdots \times \tilde{N}_d}$ it holds that,
\begin{equation}\label{eq:orthogonality_proof}
    \begin{aligned}
    &\langle \conv_s(RA,a), \conv_s(RA,b) \rangle \\
    =&\langle \tilde{\tovec} ( \conv_s(RA,a)), \tilde{\tovec} (\conv_s(RA,b))\rangle \\
    =& \langle \hat{A} \tovec (a), \hat{A} \tovec (a) \rangle   = 
    \langle \tovec (a), \tovec (b) \rangle   = 
    \langle a,b \rangle,
    \end{aligned}
\end{equation}
where we used the fact that the reordering into column vectors as well as $\hat{A}$ are orthogonal operators. Hence, we proved that $\conv_s(RA,\cdot)$ is an orthogonal operator (and in particular bijective).

\end{proof}

\begin{corollary}
Let $\theta^i \in \mathbb{R}^{\sigma \times \sigma}$ for all $i\in [d]$. For $\hat{\theta} = (\theta^1, \dots, \theta^C)$, the operator $$\hat{D}_{\hat{\theta}}: \mathbb{R}^{C \times N_1 \times \cdots \times N_d} \to \mathbb{R}^{\tilde{C}\times \tilde{N}_1 \times \cdots \times \tilde{N}_d}$$ given by 
$$\hat{D}_{\hat{\theta}}: \begin{pmatrix}x_{1,\dots} \\ \vdots \\ x_{C,\dots} \end{pmatrix} \mapsto \begin{pmatrix}
D_{\theta^1} (x_{1,\dots}) \\ 
\vdots \\ 
D_{\theta^C} (x_{C,\dots}) 
\end{pmatrix}$$
is a learnable invertible downsampling operator, parametrized by $\hat{\theta}$ over the parameter space $(\mathbb{R}^{\sigma \times \sigma})^C$, where $D_{{\theta^i}}$ is defined as in Corollary \ref{cor:learnable_downsampling_one_channel}.
\end{corollary}
\begin{proof}
By definition, it holds that $\hat{D}_{\hat{\theta}}= \oplus_{i=1}^C D_{\theta^i}$. Furthermore, since the exponential of a skew-symmetric matrix is special orthogonal, $\exp ((\theta) - (\theta)^T)$ is special orthogonal for any square matrix $\theta$ (Corollary \ref{cor:learnable_downsampling_one_channel}). Each $D_{\theta^i}$ is associated with a matrix $\oplus_{j=1}^{\tilde{N}_1 \cdots \tilde{N}_d} \exp (\theta^j - (\theta^j)^T)$ (which is orthogonal, as in the proof of Theorem \ref{thm:orthogonal_conv}), so that
$$\tilde{\tovec} \cdot \hat{D}_{\hat{\theta}} =  \left( \bigoplus\limits_{i=1}^C \bigoplus\limits_{j=1}^{\tilde{N}_1 \cdots \tilde{N}_d} \exp (\theta^j - (\theta^j)^T) \right) \cdot \tovec$$
for appropriate reorderings $\tovec : \mathbb{R}^{C\times N_1 \times \cdots \times N_d} \to \mathbb{R}^{C \cdot N_1 \cdots N_d}$ and $\tilde{\tovec} : \mathbb{R}^{\tilde{C} \times \tilde{N}_1 \times \cdots \times \tilde{N}_d} \to \mathbb{R}^{\tilde{C} \cdot \tilde{N}_1 \cdots \tilde{N}_d}$. The orthogonality proof is then exactly analogous to the proof for Theorem \ref{thm:orthogonal_conv}.
\end{proof}

\subsection{Reproduction of Known Invertible Downsampling Methods}\label{appdx:haar_by_exp}
Any symmetric matrix $\theta_\text{ps} \in \mathbb{R}^{\sigma \times \sigma}$ yields the \emph{pixel shuffle} operation, whereas
$$\theta_\text{haar} = \frac{\pi}{4} \setlength\arraycolsep{4pt}\begin{pmatrix*}[r]
0&\phantom{-}0&-1&-1 \\
0&0&1&1 \\
0&0&0&0 \\
0&0&0&0 \\
\end{pmatrix*}$$
yields the 2D Haar transform. This demonstrates that the presented technique can learn both very similar-looking as well as very diverse feature maps.

For the pixel shuffle, one can easily see that $\exp(\theta_\text{ps}-\theta_\text{ps}^T)=\exp(0)=I$ holds, which yields the corresponding matrix belonging to the pixel shuffle operation.

Proving the above Haar representation is more involved. We will show that
\begin{equation}
\exp(\theta_\text{haar}-\theta_\text{haar}^T)=\frac{1}{2}\begin{pmatrix}
 1&  \phantom{-}1& -1& -1 \\
       1&  \phantom{-}1&  \phantom{-}1&  \phantom{-}1\\
       1& -1&  \phantom{-}1& -1\\
       1& -1& -1&  \phantom{-}1
    \end{pmatrix}=:M_\text{haar} ,
    \end{equation}
which is one of the possible matrices associated to the 2D Haar transform when reordered into convolutional kernels (Theorem \ref{thm:orthogonal_conv}). We initially numerically solved $\log (M_\text{haar})$ in order to guess the representation $\theta_\text{haar}$, which we will now prove.\\

For this, we define the matrices
\begin{equation}
    A = \frac{1}{2}\begin{pmatrix}
    0 & \phantom{-}0 & -1 & -1\\
    0 & \phantom{-}0 & \phantom{-}1 & \phantom{-}1\\
    1 & -1 & \phantom{-}0 & \phantom{-}0\\
    1 & -1 & \phantom{-}0 & \phantom{-}0
    \end{pmatrix} \hspace{1cm}     B = \frac{1}{2}\begin{pmatrix}
\phantom{-}1 & -1&  \phantom{-}0&  \phantom{-}0 \\
 -1 &\phantom{-}1&  \phantom{-}0&  \phantom{-}0 \\
 \phantom{-}0 & \phantom{-}0& \phantom{-}1& \phantom{-}1 \\
 \phantom{-}0 & \phantom{-}0& \phantom{-}1& \phantom{-}1 \\
    \end{pmatrix}.
\end{equation}

Note that $\theta_\text{haar}-\theta_\text{haar}^T=\tfrac{\pi}{2}A$. Its easy to verify that $A^2=-B$, $B^2=B$ and $AB=BA=A$. Then it holds that
\begin{align}
    \exp (tA) =& \sum\limits_{n=0}^{\infty} \frac{(tA)^n}{n!} \nonumber \\
     =& \sum\limits_{n=0}^{\infty} \frac{(tA)^{2n}}{(2n)!} + \sum\limits_{n=0}^{\infty} \frac{(tA)^{2n+1}}{(2n+1)!} \nonumber\\
     =& I + \sum\limits_{n=1}^{\infty} \frac{(tA)^{2n}}{(2n)!} + \sum\limits_{n=0}^{\infty} \frac{(tA)^{2n+1}}{(2n+1)!} \label{eq:haar_trick_a}\\
     %=& I -B + B + \sum\limits_{n=1}^{\infty} \frac{(tA)^{2n}}{(2n)!} + \sum\limits_{n=0}^{\infty} \frac{(tA)^{2n+1}}{(2n+1)!} \label{eq:haar_trick_b}\\
     =& I + \sum\limits_{n=1}^{\infty} \frac{t^{2n} (A^2)^n}{(2n)!} + \sum\limits_{n=0}^{\infty} \frac{t^{2n+1} (A^2)^n A}{(2n+1)!} \nonumber\\
     =& I -B + B + \sum\limits_{n=1}^{\infty} \frac{t^{2n} (-B)^n}{(2n)!} + \sum\limits_{n=0}^{\infty} \frac{t^{2n+1} (-B)^n A}{(2n+1)!} \label{eq:haar_trick_b}\\
     =& I -B + B \cdot \left( 1+ \sum\limits_{n=1}^{\infty} (-1)^n\frac{t^{2n} }{(2n)!} \right)  +  \sum\limits_{n=0}^{\infty}(-1)^n \frac{t^{2n+1}}{(2n+1)!}BA \label{eq:haar_trick_c}\\
     =& I -B + \cos (t) B + \sin (t) A \nonumber,
\end{align}
where we were able to split the series into subseries due to the fact that the matrix exponential series is absolutely convergent. Note the similarity to Euler's formula, where $B$ corresponds to 1 and $A$ corresponds to the imaginary unit $i$. Since $(A^2)^n=B^n=B$ \emph{for all} $n$ \emph{except for} $n=0$, we had to add and then substract $B$ \eqref{eq:haar_trick_b} in order to be able to factor it out later in the $\cos$-series \eqref{eq:haar_trick_c}. Furthermore, we had to account for the $B^0=I$ that was dropped out of the $\cos$-series \eqref{eq:haar_trick_a}. The above is a special case of a formula given in \citep{gallier2003computing}.

Then, by using that $\theta_\text{haar}-\theta_\text{haar}^T=\frac{\pi}{2} A$, we see that
\begin{equation}
    \begin{aligned}
        \exp (\theta_\text{haar}-\theta_\text{haar}^T)=&\exp (A\pi/2) \\
        =&I-B+ \cos(\pi/2) B + \sin (\pi/2) A \\
                    =& I-B + 0\cdot B + A \\
                    =& M_\text{haar},
    \end{aligned}
\end{equation}
which proves our statement.

\iffalse % \CE{I'll work on the code for now and will continue if there's time at the end.}
\begin{proof}
First, we define 
\begin{equation}
    A = \begin{pmatrix}
    0 & \phantom{-}0 & -1 & -1\\
    0 & \phantom{-}0 & \phantom{-}1 & \phantom{-}1\\
    1 & -1 & \phantom{-}0 & \phantom{-}0\\
    1 & -1 & \phantom{-}0 & \phantom{-}0
    \end{pmatrix} \hspace{1cm}     B = \begin{pmatrix}
-1 & \phantom{-}1&  \phantom{-}0&  \phantom{-}0 \\
 \phantom{-}1 &-1&  \phantom{-}0&  \phantom{-}0 \\
 \phantom{-}0 & \phantom{-}0& -1& -1 \\
 \phantom{-}0 & \phantom{-}0& -1& -1 \\
    \end{pmatrix}.
\end{equation}
Its easy to verify that $A^2=2B$, $B^2=-2B$ and $AB=BA=-2A$, such that in particular, $(A^2)^n = (-2)^n B$.

\begin{equation} \begin{aligned}
    \exp (tA) =& \sum\limits_{n=0}^{\infty} \frac{(tA)^n}{n!} \\
     =& \sum\limits_{n=0}^{\infty} \frac{(tA)^{2n}}{(2n)!} + \sum\limits_{n=0}^{\infty} \frac{(tA)^{2n+1}}{(2n+1)!} \\
     =& \sum\limits_{n=0}^{\infty} \frac{t^{2n} (A^2)^n}{(2n)!} + \sum\limits_{n=0}^{\infty} \frac{t^{2n+1} (A^2)^n A}{(2n+1)!} \\
     =& \sum\limits_{n=0}^{\infty} (-1)^n \frac{t^{2n} 2^n }{(2n)!} B + \sum\limits_{n=0}^{\infty} \frac{t^{2n+1} (A^2)^n A}{(2n+1)!} \\
\end{aligned}\end{equation}
\end{proof}
\fi

\section{Implementation Details}\label{appdx:implementation_details}
When implementing invertible up- and downsampling, one needs both an implementation of the matrix exponential and for calculating gradients with respect to both the weight $\theta$ as well as the input $x$.
For the matrix exponentiation, we simply truncate the series representation
\begin{equation}
    \exp (A)=\lim_{n \to \infty}\sum\limits_{k=0}^n \frac{A^k}{k!}
\end{equation}
after a fixed number of steps. Since the involved matrices are typically small (see Remark \ref{rem:2d3dsigma}), the computational overhead of calculating the matrix exponential this way is small compared to the convolutions. More computationally efficient implementations include Pad\'{e} approximations and \emph{scaling and squaring} methods \citep{al2009computing}. 

\noindent Using $\Gamma: \theta \mapsto \theta - \theta^T$ (which is a self-adjoint, linear operator), $S:=\Gamma(\theta)$, $A:=R \cdot \exp(S)$ and $y:=\conv_s(A,x)$, employing the chain rule yields
\begin{equation}
    \begin{aligned}
        \nabla_\theta \L =& \left( \diff{y}{\theta} \right)^\ast \cdot \nabla_y \L \\
        =& \left( \diff{y}{A} \cdot \diff{A}{S} \cdot \diff{S}{\theta} \right)^\ast \cdot \nabla_y \L \\
         =& \left( \conv_s(\cdot,x) \cdot R\cdot \diff{\exp(S)}{S} \cdot \Gamma \right)^\ast \cdot \nabla_y \L. \\
    \end{aligned}
\end{equation}
The derivatives are linear operators (in the sense of Fr\'{e}chet derivatives), and as such admit adjoints. Note that the adjoint of $\conv_s(\cdot, x)$ is \emph{not} the transposed convolution (which takes values in $\mathbb{R}^{C\times {N}_1 \times \cdots {N}_d}$). Instead, this is an adjoint with respect to the kernel variable (which exists, because the convolution is linear in its kernel) and it takes values in $\mathbb{R}^{\sigma\times 1 \times s_1 \times \cdots \times s_d}$. In the following, we will denote this operator by $\conv_s^\Box$ (cf. \citep{etmann2019closer}). Furthermore, denote by $\exp'(S)$ the Fr\'{e}chet derivative of $\exp$ in $S$. When incorporating the fact that $\exp'(S)^\ast=\exp'(S^T)$ \citep{al2009computing}, this leads to the expression 
\begin{equation}
    \nabla_\theta \L = \Gamma \cdot \exp'(S^T) \cdot R^\ast \cdot \conv^\Box_s(\nabla_y \L,x).\\
\end{equation}
Analogously to the matrix exponential itself, we approximate its Fr\'{e}chet derivative by a truncation of the series
$$\exp' (S) \cdot H = \lim_{n \to \infty}\sum\limits_{k=1}^n \frac{1}{k!} M_k$$ where $M_k=M_{k-1}S+S^{k-1}M_1$ with $M_1=H$ \citep{al2009computing}. The gradients for invertible learnable upsampling follow analogously from these derivations. Both series have infinite convergence radius.
It should be noted that this implementation was mainly chosen for simplicity. More computationally efficient implementations can obtained by Pad\'{e} approximations and scaling-and-squaring algorithms.

\section{Normalizing Flows}\label{appdx:normalizing_flows}
Let the random variable $z$ have probability density function $q$, for which we will write $z \sim q(z)$. For any diffeomorphism $f$, it holds that $$x:=f^{-1}(z) \sim q(z) \cdot \left| \det \frac{\text{d} f^{-1}(z)}{\text{d} z} \right|^{-1}$$
due to the change-of-variables theorem. This means that for the probability density of $x$ (denoted $p(x)$), the log-likelihood of x can be expressed as
\begin{equation}
    \log p(x) = \log q(f(x)) + \log \left| \det \frac{\text{d} f(x)}{\text{d} x}  \right|. \label{eq:appdx_log_likelihood_normalising}
\end{equation}
If $f$ is parametrized by an invertible neural network, \eqref{eq:appdx_log_likelihood_normalising} can be maximized over a training set, which yields both a likelihood estimator $f$ as well as a data generator $f^{-1}$. Models trained this way are called \emph{normalizing flows}. The main difficulty lies in the evaluation of the determinant-term, respectively the whole 'log-abs-det' term. In the following, we will derive an expression for the determinant. 

We write $S_i := \text{split}_i$, such that $S_i^{-1} = \text{concat}_i$ (in fact even $\text{split}^\ast_i=\text{concat}_i$). As an alternative to the definition in \eqref{eq:iunet_definition}, one can define the iUNet recursively via
\begin{align}
    \Psi_i &= \Phi_i^L \circ a_{i+1} \circ \Phi_i^R \label{eq:recursive_def_psi}\\
    a_{i}  &=S_i \circ (\text{id} \oplus D_i \circ \Psi_{i} \circ U_i) \circ S_i^{-1} \label{eq:recursive_def_a},
\end{align}
where $a_{m+1}=\text{id}$ and $i$ runs from $m$ to $1$. The iUNet is then defined as $f:=\Psi_1$. 

For brevity, we write e.g. $\partial \Phi^R_i:=\diff{\Phi^R_i (x^R_i)}{x^R_i}$ for each function appearing in \eqref{eq:iunet_definition}, where the derivatives are evaluated at the respective points. Note that for linear operators, the derivatives conincide with the operators themselves.
Hence, for the derivative of $a_i$ in equation \eqref{eq:recursive_def_a} it holds that
$$\partial (\text{id} \oplus D_i \circ \Psi_{i} \circ U_i) = \text{id} \oplus D_i \cdot \partial \Psi_i \cdot U_i$$ and thus
    \begin{align}
    \det \partial a_i =& \det S_i \cdot \det(\text{id} \oplus D_i \cdot \partial \Psi_i \cdot U_i) \cdot \det S_i^{-1} \label{eq:eq_a} \\
    =& \det(\text{id} \oplus D_i \cdot \partial \Psi_i \cdot U_i) \label{eq:eq_b} \\
    =& \det(\text{id}) \cdot \det (D_i \cdot \partial \Psi_i \cdot U_i) \label{eq:eq_c}\\
    =& \det \partial \Psi_i \label{eq:eq_d},
    \end{align}
where we used that the determinant of a product is the product of the determinants \eqref{eq:eq_a}, the fact that $\det (S^{-1}_i) = 1/ \det(S_i)$ \eqref{eq:eq_b}, the fact that the determinant of a block diagonal matrix is the product of the determinants of each block \eqref{eq:eq_c} and the special orthogonality of the identity mapping and our invertible up- and downsampling operators \eqref{eq:eq_d}.
By recursion of \eqref{eq:recursive_def_psi} and \eqref{eq:recursive_def_a}, this yields
$$\det \frac{\text{d} f(x_1^L)}{\text{d} x_1^L} = \det \partial \Phi_1^L  \cdots  \det \partial \Phi_m^L \cdot  \det \partial \Phi_m^R  \cdots \det \partial \Phi_1^R ,$$
i.e. the calculation of the iUNet's determinant-term reduces to the determinants corresponding to the nonlinear invertible portions of the iUNet. This then depends on the exact parametrization of the $\Phi^L_i$ and $\Phi^R_i$, e.g. as affine coupling layers \citep{dinh2016density} or residual flow layers \citep{resflows}.

\section{Additional Information on Experimental Section} \label{appdx:experimental_details}

\subsection{Post-Processing Experiment} \label{appdx:postprocessing_details}
In the following, we will provide additional details about our artificial dataset of undersampled, low-dose CT reconstructions of the 3D 'foam phantoms' from \citep{pelt2018improving} created by filtered backprojection. Our training set consisted of 180 volumes, while the test set consisted of 20 volumes. These are comprised of cylinders of varying size, filled with a large number of holes. The volumes were generated at a resolution of $1024^3$ before trilinearly downsampling (to prevent aliasing artifacts). At a resolution of $512^3$,  a reconstruction using filtered backpropjection of a strongly undersampled parallel-beam CT projection with Poisson noise was created, which simulates a low-dose projection. A diagonal axis of the volume served as the CT axis (perturbed by angular noise). We expect the varying size of the phantoms, the artifacts on the FBP reconstructions as well as the large-scale bubble structures to favor networks with a large, three-dimensional receptive field (i.e. many downsampling operations), which justifies the use of 3D iUNets and 3D U-Nets. The FBP reconstructions as well as the ground truth volumes were downsized to $256^3$. Both 3D U-Nets as well as 3D iUNets were subsequently trained to retrieve the ground truth from the FBP reconstructions using the squared $\ell_2$-loss. The \emph{peak signal-to-noise-ratios} (PSNR) as well as the \emph{structural similarity indices} (SSIM) of this experiment cohort are compiled in Table \ref{tbl:foam_results}. Each line represents one classic 3D U-Net and one 3D iUNet of comparable size. While there is no way to construct perfectly comparable instances of both, the 3D U-Net uses 2 convolutional layers before downsampling (respectively after upsampling), whereas the 3D iUNet employs 4 additive coupling layers (each acting on half of the channels). In the case of the classic U-Net, 'channel blowup' indicates the number of channels before the first downsampling operation (identical to the number of output feature maps before reducing to one feature map again). In both architectures, layer normalization was applied. The batch size was 1 in all cases, because for the larger models this was the maximum applicable batch size due to the large memory demand. Random flips and rotations were applied for data augmentation.\newline

\subsection{Segmentation Experiment} \label{appdx:segmentation_details}
Here, we compare the gradients computed by memory-efficient backpropagation with those computed by the conventional way (i.e., all required activation of the network are \emph{stored}. For this, in Figure \ref{fig:compare_grad} we show the training and validation loss curves for both cases. The results are based on iUNet-16 architecture, and the losses are computed based on our local training/validation split of the BraTS 2018 training set. It can be seen from the figure that the memory-efficient gradient leads to a loss similar to the loss associated with the conventional gradient computation, both on the training set and validation set. 

\begin{figure}[h]
    \centering
    \begin{tabular}{ll}
       \includegraphics[width=.4\textwidth]{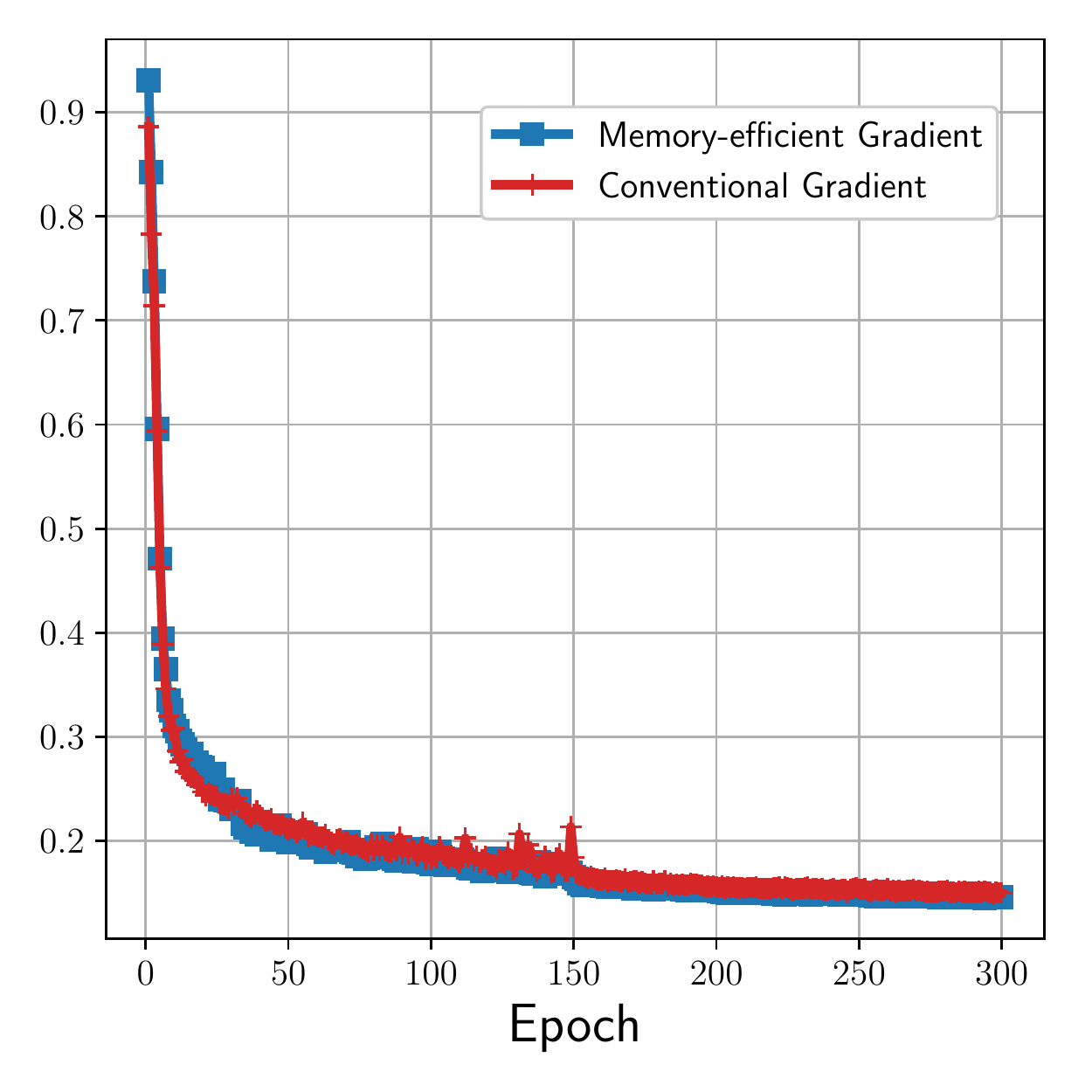}  & \includegraphics[width=.4\textwidth]{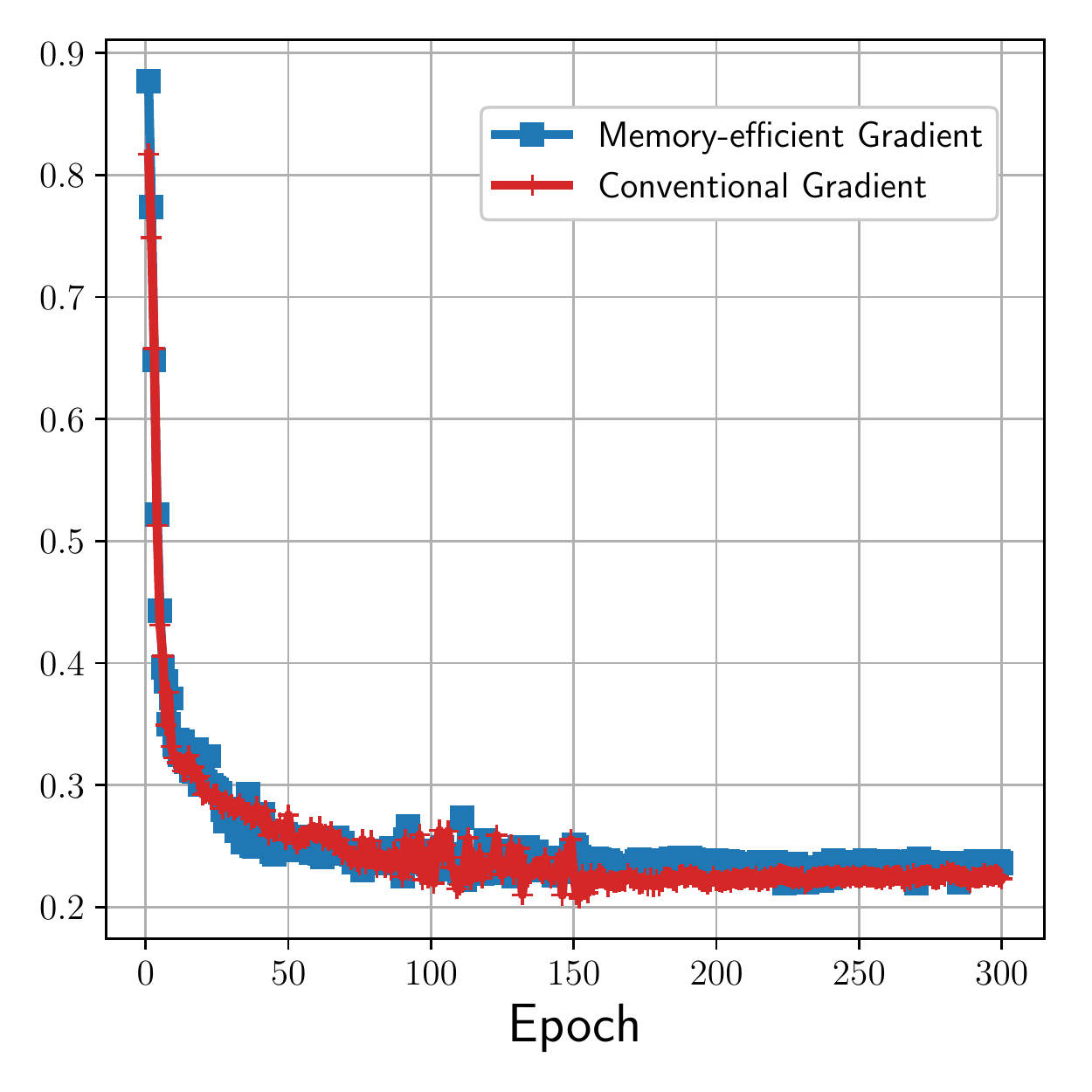}
    \end{tabular}
    \caption{Comparison between the losses for iUNet-16 implemented with memory-efficient gradient and the conventional gradient. In the left figure, the training loss is plotted against the epoch number, whereas in the right figure, the validation loss is plotted.}
    \label{fig:compare_grad}
\end{figure}

\subsection{Benchmarking Runtimes and Memory Demand}\label{appdx:runtimes}
Here, the runtimes and memory demands of weight gradients of iUNets computed with memory-efficient backpropagation are compared to conventionally computed ones. For this, we created 2D iUNets with 4 downsampling operators and slice fraction $\lambda=1/2$. Each $\Phi_i^L$ respectively $\Phi^R_i$ was defined as a sequence of $\delta \in \lbrace 5,10,20,30 \rbrace$ additive coupling layers, where each coupling block consisted of one convolutional layer with layer normalization and leaky ReLU activation functions. The memory demand and runtimes were calculated based on artificial input of size $64\times 512\times 512$ and batch size 1 (at 32 bits). The runtimes were averaged over 10 runs.\\

In Table \ref{tbl:runtime_results}, the results of this experiment are presented. While the peak memory consumption for the memory-efficient backpropagation is in practice not quite independent of depth, this may be at least partially explained by the additional overhead of storing the neural network's parameters (which was neglected in the main paper's analysis). Furthermore, some of the memory overhead may be due to the used CUDA and cuDNN backends, which is difficult to account for in practice. Still, the memory savings are quite large and become more pronounced with increasing depth, as one saves e.g. 87.8 \% of memory in the case of the deepest considered network ($\delta=30$). In terms of runtime, the memory-efficient backpropagation was between 67\% and 107\% slower than the conventional backpropagation. We stress, however, that in a real-world training scenario, the relative impact of this becomes lower, as the calculation of the gradients represents only one part of each training step, whereas e.g. the data loading or possible momentum calculation is independent of the chosen backpropagation method. 

\begin{table}[h]
\centering
\caption{Comparison of memory consumption and runtimes for calculating all weight gradients using memory-efficient (ME) backpropagation and conventional backpropagation.}
\label{tbl:runtime_results}
\begin{tabular}{cccccccc}%{rrrrrrrr}
\toprule
 & \multicolumn{3}{c}{Peak memory consumption} & & \multicolumn{3}{c}{Runtime} \\
& ME & Conventional & Ratio & & ME & Conventional & Ratio \\
\midrule
$\delta=5\phantom{0}$ & 0.85 GB & 3.17 GB & 26.8 \% & \hspace{5mm} & 1.94 s & 1.16 s & 167 \%  \\
$\delta=10$ & 1.09 GB & 5.90 GB & 18.4 \% & & 4.10 s & 2.45 s & 167 \%  \\
$\delta=20$ & 1.57 GB & 11.36 GB & 13.8 \% & & 6.82 s & 3.67 s & 186 \% \\
$\delta=30$ & 2.06 GB & 16.82 GB & 12.2 \% & & 10.63 s & 5.13 s & 207 \% \\
\bottomrule
\end{tabular}
\end{table}

\end{document}